\documentclass[11pt]{article}
\usepackage[letterpaper]{geometry}
\usepackage[parfill]{parskip}
\usepackage{amsmath,amsthm,amssymb,bbm}
\usepackage{mathtools}
\usepackage{cases}

\usepackage{microtype}

\usepackage{authblk}

\usepackage{url}
\usepackage[colorlinks,citecolor=blue,urlcolor=blue,linkcolor=blue]{hyperref}
\usepackage{cleveref}
\crefformat{equation}{(#2#1#3)}
\crefrangeformat{equation}{(#3#1#4) to~(#5#2#6)}
\crefname{equation}{}{}
\Crefname{equation}{}{}

\usepackage[authoryear]{natbib}

\newtheoremstyle{mythmstyle}
  {8 pt} 
  {3 pt} 
  {} 
  {} 
  {\bfseries} 
  {.} 
  {.5em} 
  {} 

\theoremstyle{mythmstyle}

\newtheorem{theorem}{Theorem}
\newtheorem{lemma}[theorem]{Lemma}

\newtheorem{definition}[theorem]{Definition}
\crefname{definition}{\textbf{definition}}{definitions}
\Crefname{definition}{Definition}{Definitions}
\crefname{assumption}{\textbf{assumption}}{assumptions}
\Crefname{assumption}{Assumption}{Assumptions}

\begin{document}

\title{A Generic Sample Splitting Approach for Refined Community Recovery in Stochastic Block Models}

\author[1]{Jing Lei}
 \author[2]{Lingxue Zhu}
\affil[1,2]{Department of Statistics, Carnegie Mellon University}

\maketitle

\begin{abstract}
We propose and analyze a generic method for community recovery in stochastic block models 
and degree corrected block models.  This approach can exactly recover the hidden 
communities with high probability when the expected node degrees are of order $\log n$ or 
higher.  Starting from a roughly correct community partition given by 
some conventional community recovery algorithm, this method
refines the partition in a cross clustering step.
Our results simplify and extend some of the previous work on exact community recovery, 
discovering the key role played by sample splitting. The proposed method is simple and 
can be implemented with many practical community recovery algorithms.
\end{abstract}

\section{Introduction}
\label{sec:introduction}
Stochastic block models \citep{Holland83} are a popular tool in modeling the co-occurrence of pairwise interactions between individuals in a population of interest.  In recent years, the stochastic block model and its variants, such as the degree corrected block model \citep{KarrerN11}, have been the focus of much research effort in statistics and machine learning,
with wide applications in social networks \citep{FaustW92}, biological networks and information networks \citep[see, e.g.,][]{Kemp06,BickelC09}.

Consider a network data on $n$ nodes recorded in the form of an $n$ by $n$ 
adjacency matrix $A$, where an entry 
$A_{ij}=1$ if an interaction is observed between nodes $i$ and $j$, and 
$A_{ij}=0$ otherwise.  The stochastic block model assumes that the nodes are 
partitioned into $K$ disjoint communities, and that given the community 
partition, $A_{ij}$ is an independent 
Bernoulli random variable whose probability only depends on the community 
membership of $i$ and $j$.  Such a model naturally captures the community 
structure commonly observed in network data.  It is also possible to 
incorporate node specific connectivity parameters, and the resulting 
degree corrected block model allows the network data to 
have arbitrary degree distribution.

A key inference problem in the stochastic block model and its variants is 
to recover the hidden communities from an observed adjacency matrix.  
Various methods have been developed in the last decade, including 
modularity based methods \citep{NewmanG04}, likelihood methods 
\citep{BickelC09,ZhaoLZ12}, convex optimization 
\citep{ChenSX12,LeLV14,AbbeBH14}, spectral methods 
\citep{McSherry01,Coja-Oghlan10,RoheCY11,Jin12,ChaudhuriCT12,Fishkind13,Massoulie13,LeiR14,Vu14}, and 
others \citep{Decelle11,MosselNS13,Anandkumar14}.  These methods are proved 
to be successful under different assumptions with different types of performance 
guarantee.

In this paper we propose and analyze a simple and generic sample 
splitting approach for exact community recovery in both stochastic block 
models 
and degree corrected block models.  By exact community recovery,
we mean that an algorithm can correctly recover the community membership
for all nodes with high probability.
Our method first randomly splits the nodes into two subsets and obtains
a preliminary community recovery on one subset using some conventional 
algorithm.  
Then this preliminary
community recovery is used to obtain exact recovery for the other subset of 
nodes in a cross clustering step.
 We prove that this method can exactly 
recover the communities with high probability when the expected node 
degrees are at least $C\log n$ for some constant $C$, and 
the preliminary community recovery algorithm provides a good approximation 
to the truth.

Exact recovery for stochastic block models was first studied by 
\cite{McSherry01}, where a
combinatorial projection spectral method was combined with sample splitting. 
Recently 
\cite{Vu14} modified this
method to a simpler projection based on singular value decomposition,
combined with 
multiple sample splittings. \cite{ChaudhuriCT12} extended the method
of \cite{McSherry01} to a special case of degree corrected block models
under stronger assumptions on the edge density.
Our results provide further insights to these works by showing that
the key ingredient for exact recovery is indeed 
sample splitting, which 
can be used in 
combination with a much broader class of community recovery algorithms.
In particular, the simple spectral clustering method,
which applies $k$-means to the top eigenvectors of the adjacency matrix,
is a popular method with good empirical performance.  But there is no
theory so far that guarantees its ability of exact recovery.  Our result
implies that exact recovery can be achieved by combining simple spectral 
clustering with sample splitting.

There are other methods for exact recovery in stochastic block models. 
\cite{BickelC09} proved exact recovery for a profile likelihood estimator,
for models with
 node degrees of order $\omega(\log n)$. It is 
computationally demanding to maximize the profile likelihood.
\cite{ChenSX12} studied a convex optimization method, which
requires $\Omega((\log n)^4)$ expected node degrees for exact recovery.
In contrast, 
our theory requires the expected node degrees to be at least $C\log n$
for some constant $C$. 
This matches the optimal rate 
established in \cite{AbbeBH14}, in that exact recovery is impossible when 
the maximum node degree is less than $c\log n$ for some constant $c$.    
Moreover, to the best of our knowledge, we give the first exact 
recovery method for general degree corrected block models with 
$\Omega(\log n)$ 
expected node degrees.

\section{Background}
\label{sec:prelim}
In a stochastic block model, the nodes of a network are partitioned into $K$
disjoint communities. Let $g_i\in\{1,...,K\}$ be the community label of 
node 
$i$. The observed data is an $n\times n$ symmetric binary random matrix $A$ 
with independent upper-diagonal entries $A_{ij}$ $(1\le i<j\le n)$: 
$P(A_{ij}=1)=1-P(A_{ij}=0)=B_{g_i g_j}$,  where $B\in [0,1]^{K\times K}$ is a 
symmetric 
matrix representing the community-wise edge probabilities.  For convenience we assume 
$A_{ii}=0$ for all $i$. 
Throughout this paper we assume that $K$, the number of communities, is 
known.

The community recovery problem concerns estimating the membership vector
$g=(g_1,...,g_n)$, up to a label permutation.  It is well-known that the
hardness of community recovery depends on (i) the sample size $n$, 
(ii) number of
communities $K$, (iii) differences between the rows of $B$, and (iv) the 
magnitude
of the entries in $B$, which controls the overall density of edges in the 
observed network.
In most theoretical studies of community recovery, it is common practice to
consider the large sample behavior where $n$ grows to infinity, while
other model parameters, such as $K$ and $B$, change as functions of $n$.
  In the current paper, for simplicity we focus on the 
network edge density, and fix other parameters as constants. 
Our analysis can be used to investigate dependence on other model parameters,
as discussed in \Cref{sec:discussion}.
  To this end, we assume that
\begin{equation}\label{eq:scaling}
B = \alpha_n B_0\,,
\end{equation}
where $B_0$ is a $K\times K$ non-negative symmetric constant matrix with 
maximum entry 1.  In \cite{AbbeBH14} it is shown that, for a special class 
of stochastic block models where $K=2$ and $B_0(1,1)=B_0(2,2)=1$, exact recovery is possible if and only if $\alpha_n > C\log n /n$ for a constant $C$ depending on
the off-diagonal entry $B_0(1,2)$.  For general 
stochastic block models, \cite{McSherry01} studied a spectral method with
exact recovery when $\alpha_n=\Omega((\log n)^{3.5} /n)$.  \cite{Vu14} 
improved this result to $\alpha_n\ge C\log n /n$ for a sufficiently large $C$.
Other methods, such as convex optimization \citep{ChenSX12} and profile 
likelihood \citep{BickelC09}, require stronger conditions on $\alpha_n$.

The algorithm proposed by \cite{McSherry01} first randomly splits the 
columns of $A$ into two sets. Then each set of columns are projected onto a
$K$ dimensional subspace constructed from the other set.  The final community
recovery is given by applying
a clustering algorithm to the $n\times n$ matrix obtained after
the column projection.  In the main algorithm analyzed in that paper,
the $K$ dimensional subspaces are obtained using a combinatorial approach
called CProj, which is more complicated than
some natural choices such as singular value decomposition.  
\cite{Vu14} simplified
the CProj subroutine by using multiple random splits and projection
onto singular subspaces of the split matrices.  Although this method gives
stronger theoretical guarantees, the procedure 
remains somehow cumbersome
and it is not clear how the random splitting and low dimensional projection work 
together to give exact recovery.

In this paper we argue that a single 
sample splitting can lead to exact recovery
under a much broader context.  
In particular, one can start from \emph{any} community recovery 
algorithm with approximate recovery guarantee,
and use a single sample splitting to refine the output to achieve exact recovery.
The key idea is that once we have a roughly correct community membership
for a subset of the nodes, it can be used to produce exact recovery for the
remaining nodes.  This is described in detail in Algorithm 1.
\begin{figure}[t]
  \begin{center}
\fbox{\parbox{5in}{
\begin{center}{\sf Algorithm 1:
Cross Clustering (CrossClust)}\end{center}
 \textbf{Input:} adjacency matrix $A$; subset of nodes $\mathcal V_1$; subset of nodes 
 $\mathcal V_2$; membership vector $\hat g^{(1)}$ on $\mathcal V_1$.\\
 \textbf{Require subroutine:} distance based clustering algorithm $\mathcal D$.
\begin{enumerate}
\item For each $v\in\mathcal V_2$, let $\hat h_v\leftarrow(\hat h_{v,1},...,\hat 
h_{v,K})$ with
$$
\hat h_{v,k} \leftarrow \frac{\sum_{v'\in \mathcal V_1,g^{(1)}_{v'}=k} A_{v,v'}}
{\#\{v':v'\in\mathcal V_1,g^{(1)}_{v'}=k\}}\,.
$$
\item Output $\hat g^{(2)}\leftarrow \mathcal D(\{\hat h_v:v\in \mathcal V_2\}, K)$.
\end{enumerate}
}}
\end{center}
\end{figure}

The intuition behind Algorithm 1 is very natural.  Suppose we have 
approximately correct memberships for nodes in $\mathcal V_1$. Then
we can use this membership to estimate the community-wise edge probability
for each node in $\mathcal V_2$. For $v\in\mathcal V_2$
we must have
$$
\hat h_{v,k}\approx B_{g_v,k},~~k=1,...,K\,.
$$
Therefore, if two nodes $v$ and $v'$ are in the same community, then
their corresponding $\hat h$ vectors shall both be close to the same row of $B$.
This gives us a good embedding of these nodes in a $K$ dimensional Euclidean
space with a nearly perfect clustering structure.  If such an embedding is good 
enough, 
for example, the distance between any two cluster centers is
at least four times larger than the distance between any point to its center, 
then, as pointed out in \cite{Vu14}, some classical 
distance-based clustering algorithms such as minimum spanning tree can
perfectly recover the communities.

This algorithm also reveals the need for sample splitting. In the beginning, 
it may seem
natural to obtain an approximate recovery for the entire set of nodes and
calculate the $K$-dimensional embedding for each node based on this
community partition.  This method, although may work well
in practice, is hard to analyze. The complex dependence between the numerator 
and denominator in the definition of $\hat h_{v,k}$ makes it highly challenging to
analyze the error of $\hat h_{v,k}$ as an estimator of $B_{g_v,k}$.  In contrast, 
sample 
splitting introduces more independence between these two terms and makes the analysis much more 
tractable.

Our main algorithm for stochastic block models based on sample splitting is described in Algorithm 2.
\begin{figure}[t]
  \begin{center}
\fbox{\parbox{5in}{
\begin{center}{\sf Algorithm 2:
Community Recovery by Sample Splitting (SplitClust)}\end{center}


 \textbf{Input:} adjacency matrix $A$; number of communities $K$.\\
 \textbf{Require subroutine:} \textsf{CrossClust}; initial community recovery 
 algorithm $\mathcal S$.
\begin{enumerate}
\item Randomly split the nodes into two equal sized sets, $\mathcal V_1$ and 
$\mathcal V_2$.
\item $\hat g^{(1)}\leftarrow\mathcal S(A^{(1)}, K)$, where $A^{(1)}$ is 
the induced adjacency matrix over $\mathcal V_1$.
\item $\hat g^{(2)}\leftarrow{\sf CrossClust}(A, \mathcal V_1, \mathcal V_2, \hat g^{(1)})$.
\item Output $\hat g\leftarrow{\sf CrossClust}(A, \mathcal V_2, [n], \hat g^{(2)})$.
\end{enumerate}
}}
\end{center}
\end{figure}
The initial community recovery algorithm $\mathcal S$ can be chosen by the user.  
As shown in 
\Cref{sec:main-results} below, it only needs to satisfy some mild accuracy
requirement, which can be achieved by many popular and practical methods, such
as spectral methods and likelihood based methods.

\section{Main results for stochastic block models}
\label{sec:main-results}
\paragraph{Terminology}
Throughout this paper, the term ``with high probability'' means
 ``with probability at
least $1-O(n^{-1})$''.
For two community membership vectors $g$ and $\hat g$, we say $\hat g$ makes $m$ recovery errors as an estimate 
of $g$, where $m$ is the smallest integer such that there exists a label permutation on $\hat g$ under 
which $\hat g$ and $g$ disagree
at exactly $m$ entries.  We write $\hat g=g$ if $\hat g$ makes zero error.
For a membership vector $g$ on $[n]$, and a subset $\mathcal V\subseteq [n]$, 
$g^{(\mathcal V)}$
denotes the membership vector obtained by confining $g$ on $\mathcal V$. We
use $g^{(j)}$ in place of $g^{(\mathcal V_j)}$ $(j=1,2)$ for simplicity.
We use $\| \cdot \|$ to denote the $\ell_2$ norm of vectors in Euclidean spaces. 
For any matrix $A$, we refer to its $(i,j)$th element as $A(i,j)$, and its $i$-th row 
as $A(i, \cdot)$. Sometimes $A_{i,j}$ is used in place of $A(i,j)$ for brevity.

To focus on the network edge sparsity, we assume that the smallest community size is proportional to $n$.
\begin{definition}[Proper membership]\label{def:proper}
Given a subset $\mathcal V\subseteq [n]$, a membership vector $g$ on $\mathcal V$, and a positive constant $\pi_0\in(0,1/K]$, we say $g^{(\mathcal V)}$ is \emph{$\pi_0$-proper} 
if $\min_{1\le k\le K}|\{i\in\mathcal V:g_i=k\}|\ge\pi_0 n$.  
\end{definition}

  We make the following assumptions.
\begin{enumerate}
  \item [(A1)] In the network sparsity scaling in \eqref{eq:scaling},
 $B_0$ is a constant matrix such that the minimum $\ell_2$ difference between
two rows of $B_0$ is at least $\gamma$, a positive constant.
\item [(A2)] The true community membership $g$ on $\mathcal V = [n]$ is
$\pi_0$-proper for some constant $\pi_0>0$.
\item [(A3)] The initial community recovery algorithm $\mathcal S$, with high probability, 
has recovery error at most 
$n/f(n\alpha_n)$ when $\alpha_n=\Omega(\log n / n)$, where $f$, possibly depending on
$\pi_0$ and $B_0$, satisfies 
$\lim_{x\rightarrow\infty}f(x)\uparrow\infty$.
\end{enumerate}

Assumption A1 puts a lower bound on pairwise difference between the rows of $B_0$, 
which is a minimum requirement for the communities to be distinguishable.  This also implies that $K$ is 
a fixed 
constant.  Assumption A2 puts a lower bound on the minimum community size. These are mainly 
for
simplicity of
presentation, so that we can focus on the dependence on the network sparsity.  Our 
argument does allow
for some mild generalizations so that both the number of communities and the minimum 
community size can
change with $n$ in a non-trivial manner, as discussed in \Cref{sec:discussion}.  
Assumption A3 puts a mild requirement on the 
accuracy of
the initial community recovery algorithm.  That is, with high probability the initial algorithm $\mathcal S$
correctly recovers the membership of all but a vanishing
proportion of nodes, as the expected node degrees grow at $\Omega(\log n)$ rate or faster.  
It can be satisfied by some very simple and 
practical methods.
For example, the simple spectral clustering method, which applies $k$-means to the rows 
of leading eigenvectors
of the adjacency matrix, satisfies Assumption A3 with 
$f(n\alpha_n)=L(\pi_0,B_0) 
n\alpha_n$ for some function $L>0$ independent of $n$ when $B_0$ has full rank  \citep{LeiR14}.  Our 
theoretical developments will
focus on the case where $n$ is large enough
such that $f(n\alpha_n)$ dominates any constant quantities involved in the analysis.

In the following analysis, we consider Algorithms 1 and 2 with subroutine $\mathcal D$ being minimum
spanning tree clustering, which constructs a minimum spanning tree on the input data and removes the
$K-1$ largest edges. We have the following performance guarantee of Algorithm 1.

\begin{theorem}[Exact recovery using sample splitting]\label{thm:main-sbm}
  Given a membership vector $g$ and connectivity matrix $B=\alpha_n B_0$ satisfying Assumptions A1, A2, 
  let $\hat g$ be the output of \textsf{SplitClust} (Algorithm 2) with
input matrix $A$ generated by the corresponding stochastic block model
and subroutine $\mathcal S$ satisfying Assumption A3.
There exists a constant $C$ such that if $\alpha_n\ge C\log n / n$ then with high probability we
 have $\hat g=g$.
\end{theorem}
\Cref{thm:main-sbm}  essentially says that
when the true community partition is balanced and the expected node degrees are sufficiently
 larger than $\log n$, one can always use sample splitting to refine 
an approximately correct
community recovery algorithm to achieve exact recovery. The proof, given in \Cref{sec:proof-sbm},
consists of two 
simple applications of 
the key result \Cref{lem:crossclust} below, which guarantees the accuracy of 
subroutine \textsf{CrossClust}.
\begin{lemma}[Accuracy of \textsf{CrossClust}]\label{lem:crossclust}
  Let $A$ be an adjacency matrix generated by a stochastic block model satisfying
  the assumptions of \Cref{thm:main-sbm}, and $\mathcal V_1$ be a subset with $\pi_0$-proper membership vector $g^{(1)}$. Let
  $\hat g^{(1)}$ be an estimated membership vector on $\mathcal V_1$ independent of the edges between
  $\mathcal V_1$ and $\mathcal V_2$, with recovery error at most $|\mathcal V_1|/f(|\mathcal V_1|\alpha_n)$.
  There exists a constant $C$ such that if $\alpha_n \ge C \log n / n$, then with high probability, $\hat g^{(2)}$, the output of 
  Algorithm 1, satisfies
  $\hat g^{(2)}=g^{(2)}$.
\end{lemma}
\Cref{lem:crossclust} ensures that the subroutine \textsf{CrossClust} produces
exact recovery on $\mathcal V_2$ with high probability.
The probabilistic claim in \Cref{lem:crossclust} is indeed conditional on 
given $\hat g^{(1)}$. Here we do not emphasize the conditional nature
of this result as the randomness is from edges between $\mathcal V_1$ and $\mathcal V_2$, and hence is independent of $\hat g^{(1)}$ by assumption.
The proof of \Cref{lem:crossclust}, as detailed
in \Cref{sec:proof-sbm}, is based on a
careful decomposition of estimation error $|\hat h_{v,k}-B(g_v,k)|$ followed by
large deviation bounds.

\section{Extension to degree corrected block models}
\label{sec:dcbm}
The degree corrected block model \citep{KarrerN11} extends the stochastic block model by introducing 
additional node level
degree heterogeneity.  In addition to the membership vector $g$ and community-wise connectivity matrix
$B$, the degree corrected block model incorporates a parameter $\psi\in(0,1]^n$ to model the node level
activeness. Then the edge $A_{ij}$ between nodes $i$ and $j$ is an independent 
Bernoulli variable
with parameter 
$\psi_i\psi_j B_{g_i g_j}$.  For identifiability, we assume
$\max_{i: g_i = k}\psi_i=1$, for all $1\le k\le K$.  The parameter $\psi_i$ reflects the relative
activeness of node $i$ in its community.  The degree corrected block model is able to model a much
wider range of network data and is more realistic than the regular stochastic block model. There are
relatively fewer results on exact recovery for degree corrected block models.
\cite{ZhaoLZ12} extended the result of \cite{BickelC09}, showing that the profile likelihood estimator
can recover exactly when $\alpha_n=\omega(\log n / n)$.  \cite{ChaudhuriCT12} extended the method of
\cite{McSherry01} to a special case of degree corrected models with a stronger
requirement on the decay rate of $\alpha_n$.  
In the following
we show that the simple sample splitting method can be successful under general
degree corrected block models when $\alpha_n\ge C\log n/n$ for sufficiently
large constant $C$.

Under the degree corrected block model, we need to modify the 
\textsf{CrossClust} algorithm so that the effect of nuisance parameter $\psi$ is cancelled out by
a normalization step.  To this end, we introduce the spherical cross clustering algorithm
in Algorithm 1'.

\begin{figure}[t]
  \begin{center}
\fbox{\parbox{5in}{
\begin{center}{\sf Algorithm 1':
Spherical Cross Clustering (CrossClustSphere)}\end{center}
 \textbf{Input:} adjacency matrix $A$; subset of nodes $\mathcal V_1$; subset of nodes $\mathcal V_2$; 
 membership vector $g^{(1)}$ on $\mathcal V_1$.\\
 \textbf{Require subroutine:} 
 distance based clustering algorithm $\mathcal D$.
\begin{enumerate}
\item For each $v\in\mathcal V_2$, let $\hat h_v=(\hat h_{v,1},...,\hat h_{v,K})$ be 
the same as
give in Step 1 of Algorithm 1 (\textsf{CrossClust}).
\item Output $\hat g^{(2)}\leftarrow \mathcal D(\{\hat h_v/\|\hat h_v\|:v\in \mathcal V_2\}, K)$.
\end{enumerate}
}}
\end{center}
\end{figure}

The exact recovery property of the sample splitting approach  can be established for degree corrected
block models under essentially the same conditions as for the regular stochastic block models, with a modified community separation condition and
an additional condition on the nuisance parameter $\psi$.  Recalling that
$B_0(k,\cdot)$ denotes the $k$th row of $B_0$, we assume the following.
\begin{enumerate}
  \item [(A1')] The minimum $\ell_2$ difference between two normalized-rows of $B_0$ is at least $\tilde{\gamma}$, a positive constant:
  $$
  \min_{1\le k< k'\le K}\left\|\frac{B_0(k,\cdot)}{\|B_0(k,\cdot)\|}-
  \frac{B_0(k',\cdot)}{\|B_0(k',\cdot)\|}\right\| \ge \tilde \gamma>0\,.
  $$
  \item [(A4)] $\min_{1\le i\le n}\psi_i\ge \psi_0$ for some constant $\psi_0\in(0,1]$.
\end{enumerate}
Assumption A1' modifies Assumption A1 to account for the normalization step in \textsf{CrossClustSphere}, which is necessary for degree corrected block models
because two rows in $B$ differing only by a constant scaling are indistinguishable
due to the node activeness parameter.   Assumption A4 prevents any node from being too inactive, otherwise there will be too few edges for that node,
making exact recovery unlikely.  Under these assumptions, the spherical spectral clustering method 
described and analyzed
in \cite{LeiR14} satisfies Assumption A3 with $f(n\alpha_n)=L'(\pi_0, B_0)\sqrt{n\alpha_n}$, provided that $B_0$ has full rank.

\begin{theorem}[Exact recovery for degree corrected block models]\label{thm:main-dcbm}
   Let $A$ be an adjacency matrix generated from a degree corrected block model with 
   membership vector $g$, connectivity matrix $B=\alpha_n B_0$, and node activeness 
   vector
    $\psi$ satisfying Assumptions A1', A2 and A4. 
    Let $\hat g$ be the output of \textsf{SplitClust} (Algorithm 2) using
    subroutine \textsf{CrossClustSphere} (Algorithm 1') instead of \textsf{CrossClust}
    and initial recovery algorithm $\mathcal S$ satisfying Assumption A3.
  There exists a constant $C$ such that if $\alpha_n\ge C\log n / n$ then
  $\hat g=g$ with high probability.
\end{theorem}

The proof of \Cref{thm:main-dcbm} is similar to that
of \Cref{thm:main-sbm}, and uses the following analogous result of \Cref{lem:crossclust}.  Detailed proofs of both results are given in
\Cref{sec:proof-dcbm}.

\begin{lemma}[Accuracy of \textsf{CrossClustSphere}]\label{lem:crossclustsphere}
  Let $A$ be an adjacency matrix generated by a degree corrected block model satisfying
  the assumptions of \Cref{thm:main-dcbm}, and $\mathcal V_1$ be a 
  subset with $\pi_0$-proper membership vector $g^{(1)}$. Let
  $\hat g^{(1)}$ be a membership vector on $\mathcal V_1$ independent of the edges between
  $\mathcal V_1$ and $\mathcal V_2$, with recovery error at most $|\mathcal V_1|/f( 
  |\mathcal V_1|\alpha_n)$. 
   There exists a constant $C$ such that if $\alpha_n\ge C\log n / n$, then with high probability, $\hat g^{(2)}$, the output of \textsf{CrossClustSphere}
  (Algorithm 1'), satisfies
  $\hat g^{(2)}=g^{(2)}$.
\end{lemma}

\section{Discussion}
\label{sec:discussion}
\paragraph{Dependence on other model parameters.} 
In this paper we investigate exact community recovery with a special focus on the
overall network sparsity (e.g., the rate at which $\alpha_n$ vanishes).
In the study of stochastic block models, the effect of other model parameters, such as 
the number of communities and
community size imbalance may also be of interest.
  The arguments
developed in this paper, for example, \Cref{lem:individualhv},
do keep track of these parameters and hence can be used to study 
the more general scenario where these parameters are also allowed to change 
with $n$ non-trivially.   
In particular, when $K=2$ with $\mathcal I_1$, $\mathcal I_2$
being the two communities,
one can show that the sample splitting approach succeeds with high probability
when $\alpha_n=\Omega(1)$ and $\min(|\mathcal I_1|,|\mathcal I_2|)\ge C\sqrt{n}$
for large enough $C$.

\paragraph{Extension to V-fold cross clustering}  Our main algorithm, Algorithm 2,
splits the nodes into two subsets.
Practically, halving the set of nodes may result in a non-negligible loss
of accuracy in the initial estimation of $g^{(1)}$.
 This issue can be mitigated by extending the sample splitting to a 
V-fold cross clustering.  Given positive integer $V\ge 2$, we can split the nodes
into $V$ subsets, $\mathcal V_1,...,\mathcal V_V$ of equal size.  For each $j=1,...,V$,
let $\mathcal V_{-j}=\cup_{j'\neq j} \mathcal V_{j'}$.  Then we can obtain a
preliminary community estimate $\hat g^{(-j)}$ for $\mathcal V_{-j}$, and then apply
\textsf{CrossClust} to $(A, \mathcal V_{-j}, \mathcal V_j, \hat g^{(-j)})$ to obtain 
community estimate $\hat g^{(j)}$ for $\mathcal V_j$.  Our theoretical
results can be extended to this case, suggesting that
$\hat g^{(j)}=g^{(j)}$ with high probability under appropriate conditions.
After obtained $\hat g^{(j)}$ for all $1\le j\le V$, one can merge these membership 
vectors by matching the edge frequencies between communities in different
subsets.  It has been empirically observed that such V-fold cross clustering with $V> 2$
usually outperforms the sample splitting method.

\paragraph{Open problems} The only step in our proof that requires sample splitting
is the large deviation bound for the term $T_1$ in the proofs of
\Cref{lem:individualhv,lem:DChvbound}, where the summation of $A_{v,v'}$ is over a random set
$\{v'\in\mathcal V_1:\hat g^{(1)}_{v'}\neq g^{(1)}_{v'}\}$.  The sample splitting
makes the summand $A_{v,v'}$ independent of this index set, allowing us to condition on the
index set and apply
Bernstein's inequality.  As mentioned earlier in \Cref{sec:prelim}, a natural 
alternative is to obtain a preliminary community estimate for the entire set of nodes,
and then cross cluster each node using this preliminary community partition.  
That is,
we use $\mathcal V_1=\mathcal V_2=[n]$ in Algorithm 1.  
Such a self-cross clustering approach, although 
hard to analyze, gives
very competitive practical performance.  As an example,
consider the political blog data \citep{PolBlog}, where the edges represent
hyperlinks among 1222 weblogs on U.S. politics in 2004, and two communities
are recognized as ``liberal'' and ``conservative'', respectively.  
It is widely believed that
the degree corrected block model with $K=2$ fits the data well 
\citep{ZhaoLZ12,Jin12,Yan14}.
The self-cross clustering method makes about 10\% fewer errors than both the simple
spherical spectral clustering \citep{LeiR14} and the V-fold cross clustering considered
in this paper.  It would be an interesting open problem to provide rigorous performance
guarantee for the self-cross clustering method.

\section{Proof of main results}
\label{sec:proof}
We introduce several more notations. 
For any $i\in\{1,2\}, k \in \{1,2,...,K\}$, we denote
\begin{equation*}
 \mathcal I_k = \left\{ v: g_v = k \right\}\,, \ 
 \mathcal I_k^{(i)} = \left\{ v: v \in \mathcal V_i, \, g_v^{(i)} = k \right\}\,, \  
 \hat{\mathcal I}_{k}^{(i)} = \left\{ v: v \in \mathcal V_i, \, \hat g_v^{(i)}  = k \right\}\,.
\end{equation*}
Usually the true membership $g$ and estimated membership $\hat g$ agree on most 
entries up to
a label permutation. For simplicity we will assume, without loss of generality, that the
permutation is identity.

 \subsection{Proofs for stochastic block models}\label{sec:proof-sbm}
 The main theorem for stochastic block models follows from two simple applications
 of the accuracy of cross clustering (\Cref{lem:crossclust}).
\begin{proof}[Proof of \Cref{thm:main-sbm}]
Note that
\[
 \mathbb{P} \left( \hat{g} = g \right) \geq \mathbb{P} \left( \hat{g}=g, \hat{g}^{(2)}  = g^{(2)} \right) = \mathbb{P} \left( \hat{g} = g \middle| \hat{g}^{(2)} = g^{(2)} \right) \mathbb{P} \left( \hat{g}^{(2)} = g^{(2)} \right)\,.
\]

By \Cref{lem:propersubset}, $g^{(1)}$ is $\pi_0/3$-proper with high probability. Then by  \Cref{lem:crossclust}, when $\alpha_n \ge C \log n / n$ for some constant $C$, $\hat{g}^{(2)} = g^{(2)}$ with high probability.

Finally, when $\hat{g}^{(2)} = g^{(2)}$, $\hat{g}^{(2)}$ is $\pi_0/3$-proper with high
probability according to \Cref{lem:propersubset}. Then by \Cref{lem:crossclust} again, $\hat{g} = g$ with high probability when $\alpha_n > C' \log n / n$ for another large enough constant $C'$.
\end{proof}

\begin{proof}[Proof of \Cref{lem:crossclust}]
Suppose $\alpha_n = a \log n / n$ for some constant $a$.
If for all nodes $v \in \mathcal V_2$, for some constant $\delta>0$,
\begin{equation}\label{eq:h-error}
\|\hat{h}_v - B(g_v, \cdot) \| \leq \delta \frac{\log n}{n}\,,
\end{equation}
then we have the following separation conditions
 \[\begin{split}
 & \sup_{v,v' \in \mathcal V_2, g_v = g_{v'}} \| \hat{h}_v - \hat{h}_{v'}  \| \leq 2 
 \delta \frac{\log n}{n}\,,\\
 &  \inf_{v,v' \in \mathcal V_2, g_v \neq g_{v'}} \| \hat{h}_v - \hat{h}_{v'}  \| \geq 
 \inf_{1 \leq k < k' \leq K} \| B(k, \cdot) - B(k', \cdot)\| - 2 \delta \frac{\log 
 n}{n}  \geq ( a \gamma - 2 \delta) \frac{\log n}{n}\,.
 \end{split} \]
The distance based clustering subroutine $\mathcal D$ used in Algorithm 1, such
as the minimum spanning tree, can correctly cluster all nodes, i.e., $\hat{g}^{(2)} = 
g^{(2)}$, if
 \[ \sup_{v,v' \in \mathcal V_2, g_v = g_{v'}} \| \hat{h}_v - \hat{h}_{v'}  \|  <  
 \inf_{v,v' \in \mathcal V_2, g_v \neq g_{v'}} \| \hat{h}_v - \hat{h}_{v'}  \| \,. \]

Therefore, it suffices to show that with high probability, $ \| \hat{h}_v - B(g_v, 
\cdot) \| \leq \delta \frac{\log n}{n}$ for all $v \in \mathcal V_2$, where 
\begin{equation} 0 < \delta < \frac{a \gamma}{4}  \,. \label{ineq:delta}\end{equation}

 By \Cref{lem:individualhv}, the approximation bound \eqref{eq:h-error}
 and inequality \eqref{ineq:delta} hold
 with high probability
 if we choose $q$ in \Cref{lem:individualhv} 
 to be a constant larger than $4/\gamma$.
\end{proof}

The following lemma establishes our key observation that $\hat{h}_v$ should be close to the corresponding row of connectivity matrix, $B(g_v, \cdot)$. 

\begin{lemma}\label{lem:individualhv}
Given $\mathcal V_1$, $\mathcal V_2$, $g^{(1)}$, and $\hat g^{(1)}$ satisfying the conditions of \Cref{lem:crossclust} with $\alpha_n=q \delta \log n /n$. For any $q>1$, there exists $\delta_0=\delta_0(\pi_0,c_0,K,q)$, 
such that for all $\delta\ge \delta_0$
\[\mathbb{P} \left(\| \hat{h}_v - B(g_v, \cdot) \|\le \delta \frac{\log n}{n},~~\forall~v\in\mathcal V_2\right) =1-O(n^{-1})\,, \]
where the probability is conditional on  $\mathcal V_1$, $\mathcal V_2$ and $\hat g^{(1)}$.
 
\end{lemma}

\begin{proof} Let $a=q\delta$.
Because
\[ \mathbb{P} \left(\| \hat{h}_v - B(g_v, \cdot) \| \ge \delta \frac{\log n}{n} \right) 
\leq \sum_{k=1}^K \mathbb{P} \left( \left| \hat{h}_{v,k} - B(g_v, k) \right| \ge \frac{ 
\delta}{\sqrt{K}} \frac{\log n}{n} \right)\,, \]
it suffices to bound all $K$ coordinates individually.
Now, for any $k \in \{1,2,.., K\}$, since $ \lim_{n \to \infty} f(\pi_0 a \log n) = \infty$, when $n$ is large enough, we have
\begin{align}\label{ineq:I_kbound} 
&  \sum_{l \neq k}\left| \hat{\mathcal I}_{l}^{(1)} \cap \mathcal I_k^{(1)} \right| \leq |\{v':\hat g^{(1)}_{v'}\neq g^{(1)}_{v'}\}|\le \frac{ |\mathcal V_1| }{f(|\mathcal V_1|\alpha_n)} \leq \frac{n}{f(\pi_0 a \log n)}\,. \\
 & \left| \hat{ \mathcal I}_k^{(1)} \right| \geq \left|  \mathcal I_k^{(1)}  \right| - \sum_{l \neq k}\left| \hat{\mathcal I}_{l}^{(1)} \cap \mathcal I_k^{(1)} \right| \geq \left[ \pi_0 - \frac{1}{f (\pi_0 a \log n) }  \right] n \geq \pi_0 n / 2\,. \label{eq:I_hat_bound}
\end{align}
For any $1\le k\le K$, we have
\[\begin{split}
\left| \hat{h}_{v,k} -  B(g_v, k) \right| & \leq \left| \frac{ \sum_{ v' \in 
\hat{\mathcal I}_k^{(1)} } A_{v,v'} - \sum_{v' \in \mathcal I_k^{(1)}} A_{v, v'} }{ 
|\hat{\mathcal I}_k^{(1)}| } \right| + \left| \frac{ \sum_{ v' \in \mathcal I_k^{(1)} 
} A_{v,v'} }{ |\hat{\mathcal I}_k^{(1)}| } -  \frac{ \sum_{ v' \in \mathcal I_k^{(1)} 
} A_{v,v'} }{ |\mathcal I_k^{(1)}| } \right| \\ 
& \qquad + \left|  \frac{ \sum_{ v' \in \mathcal I_k^{(1)} } A_{v,v'} }{ |\mathcal 
I_k^{(1)}| }  - B(g_v, k) \right| \\
& \leq \frac{ \left| \sum_{ v' \in \hat{\mathcal I}_k^{(1)} } A_{v,v'} - \sum_{v' \in \mathcal I_k^{(1)}} A_{v, v'} \right| }{ \pi_0 n/2 } + \frac{ \left|  | \mathcal I_k^{(1)}| -  |\hat{\mathcal I}_k^{(1)}|  \right| }{ \pi_0^2 n^2 /2 }  \sum_{ v' \in \mathcal I_k^{(1)} } A_{v,v'} \\
& \qquad +  \left|  \frac{ \sum_{ v' \in \mathcal I_k^{(1)} } A_{v,v'} }{ |\mathcal I_k^{(1)}| }  - B(g_v, k) \right|  \\
& = T_1+T_2+T_3\,.
\end{split}\]
Thus
\[
 \mathbb{P}  \left(  \left| \hat{h}_{v,k} - B(g_v, k) \right| \ge \frac{ \delta}{\sqrt{K}} \frac{\log n}{n}  \right) 
  \leq\sum_{j=1}^3 \mathbb{P} \left( T_j \ge \frac{\delta}{3\sqrt{K}} \frac{\log n}{n} \right)\,.\]
Now we only need to bound the three terms individually. 
First by the assumption on $\hat g^{(1)}$ we know
$|\{v': \hat g^{(1)}_{v'}\neq g^{(1)}_{v'}\}|\le n/f(\pi_0 a \log n)$.
Then using Bernstein's inequality we have for $n$ large enough,
\[\begin{split}
\mathbb{P} \left( T_1 \ge \frac{\delta}{3\sqrt{K}} \frac{\log n}{n} \right) & \leq \mathbb{P} \left( \sum_{v': \hat g^{(1)}_{v'}\neq g^{(1)}_{v'}} A_{v, v'} \ge 
\frac{\pi_0 n}{2}\frac{\delta}{3\sqrt{K}} \frac{\log n}{n} \right)  \\
& \leq n^{-\frac{c_1^2(\pi_0,K)\delta^2}{a/f(\pi_0 a \log n)+c_1(\pi_0,K)\delta/3}}\le n^{-c_1(\pi_0,K)\delta}\,,
\end{split}\]
where $c_1(\pi_0,K)$ is a positive constant depending only on $\pi_0$ and $K$.

To control $T_2$, similarly we have for $n$ large enough using Bernstein's
inequality
\[\begin{split}
\mathbb{P} &\left( T_2 \ge \frac{\delta}{3\sqrt{K}} \frac{\log n}{n} \right) \leq \mathbb{P} \left( 
  \sum_{ v' \in \mathcal I_k^{(1)} } A_{v,v'} \ge
  \frac{\pi_0^2 n^2}{2\left|  | \mathcal I_k^{(1)}| -  |\hat{\mathcal I}_k^{(1)}|  \right|}  \frac{\delta}{3\sqrt{K}} \frac{\log n}{n}\right) \\
& \qquad \le  \mathbb{P} \left( 
  \sum_{ v' \in \mathcal I_k^{(1)} } A_{v,v'} \ge
  \frac{\pi_0^2 \delta \log n f(\pi_0a\log n)}{6\sqrt{K}}  \right) \\
& \qquad \leq n^{-\frac{c_2^2(\pi_0,K)f^2(\pi_0 a\log n)\delta^2}{a+c_2(\pi_0,K)f(\pi_0 a\log n)\delta/3}}\le n^{-c_2(\pi_0,K)\delta}\,,
\end{split}\]
where $c_2(\pi_0,K)$ is another positive constant function of $\pi_0$ and $K$.

 Directly applying Bernstein's inequality to $T_3$, we have
\[\begin{split}
 \mathbb{P} \left( T_3  \ge \frac{\delta}{3\sqrt{K}} \frac{\log n}{n} \right) &  
 \le \mathbb{P} \left( \bigg|\sum_{ v' \in \mathcal I_k^{(1)} } [ A_{v,v'}  - B(g_v, k) ]\bigg|\ge  \frac{\delta}{3\sqrt{K}}\frac{\log n}{n} \pi_0 n \right) \\
 & \leq 2 n^{-c_3(\pi_0,K,q)\delta} \,,
 \end{split}\,,\]
where $c_3(\pi_0,K,q)$ is a positive constant function independent of $n$.

The claimed result follows by choosing $\delta$ large enough and applying union
bound over all $v\in\mathcal V_2$.
\end{proof}

We conclude this subsection by establishing probability lower bounds of
having proper subsets under random splitting.
\begin{lemma}[Probability of having proper split subsets]\label{lem:propersubset}
If the true membership vector $g$ on $[n]$ is $\pi_0$ proper, and $[n]$ is 
randomly split into two equal-sized subsets $\mathcal V_1, \mathcal V_2$
with corresponding $g^{(1)}$, $g^{(2)}$. Then $g^{(1)}$ and 
$g^{(2)}$ are $\pi_0/3$-proper with high probability.
\end{lemma}

\begin{proof}
By definition we have $|\mathcal V_1|=|\mathcal V_2|\ge n/2$.
The claimed result follows easily from an exponential tail probability bound for
hypergeometric random variables \citep[see, e.g.,][]{Skala13}, for $i=1,2$,
\[ \mathbb{P} \left(   | \mathcal I_k^{(i)} | < \pi_0 n/3 \right) \leq  
\mathbb P \left( | \mathcal I_k^{(i)} |  - \mathbb{E} | \mathcal I_k^{(i)} | < - \pi_0 n/6 \right)  \leq e^{ - \pi_0^2 n/18} \,. \]
\end{proof}

\subsection{Proofs for degree corrected block models}\label{sec:proof-dcbm}
In the following proofs, we denote $\tilde{B}$ as the $K \times K$ weighted connectivity matrix, where
 \begin{equation}\label{def:tildeB}
  \tilde{B}(i, j) =  \frac{ \sum_{v' \in \mathcal I_j^{(1)}} \psi_{v'}}{ | \mathcal I_j^{(1)} | } B(i,j)\,.
  \end{equation}

\begin{proof}[Proof of \Cref{thm:main-dcbm}]
Note that
\[
 \mathbb{P} \left( \hat{g} = g \right) \geq \mathbb{P} \left( \hat{g}=g, \hat{g}^{(2)}  = g^{(2)} \right) = \mathbb{P} \left( \hat{g} = g \middle| \hat{g}^{(2)} = g^{(2)} \right) \mathbb{P} \left( \hat{g}^{(2)} = g^{(2)} \right) \,.
\]

By Lemma \ref{lem:propersubset}, $g^{(1)}$ is $\pi_0/3$-proper with high probability. Then by Lemma \ref{lem:crossclustsphere}, when $\alpha_n > C \log n / n$ for some constant $C$, $\hat{g}^{(2)} = g^{(2)}$ with high probability. Therefore, $\hat{g}^{(2)} = g^{(2)}$ with high probability.

Finally, when $\hat{g}^{(2)} = g^{(2)}$, it is $\pi_0/3$-proper with high probability according to \Cref{lem:propersubset}. Then by Lemma \ref{lem:crossclustsphere} again, $\hat{g} = g$ with high probability when $\alpha_n > C' \log n / n$ for another large enough constant $C'$.
\end{proof}

\begin{proof}[Proof of \Cref{lem:crossclustsphere}]
Suppose $\alpha_n = a \log n / n$ for some constant $a$. If  for all nodes $v \in \mathcal V_2$, for some constant $\delta>0$,
\begin{equation}\label{eq:DCh-error}
 \left\| \frac{\hat{h}_v}{ \| \hat{h}_v \| } - \frac{ \tilde{B}(g_v, \cdot) }{ \| \tilde{B}(g_v, \cdot)\|} \right\| \leq \delta\,,
 \end{equation}
then we have the following separation conditions
 \[\begin{split}
 & \sup_{v,v' \in \mathcal V_2, g_v = g_{v'}} \left\| \frac{\hat{h}_v}{ \| \hat{h}_v \| } -\frac{ \hat{h}_{v'} }{ \| \hat{h}_{v'} \| }\right\| \leq 2 \delta \,, \\
 &  \inf_{v,v' \in \mathcal V_2, g_v \neq g_{v'}} \left\| \frac{\hat{h}_v}{ \| \hat{h}_v \| } - \frac{ \hat{h}_{v'} }{ \| \hat{h}_{v'} \| } \right\| \geq \inf_{1 \leq k < k' \leq K} \left\| \frac{ \tilde{B}(k, \cdot) }{ \| \tilde{B}(k, \cdot)\|} - \frac{ \tilde{B}(k', \cdot) }{ \| \tilde{B}(k', \cdot)\|} \right\| - 2 \delta \,. 
 \end{split} \]
We know from \Cref{lem:DCrowdiff} that
\[ \inf_{1 \leq k < k' \leq K} \left\| \frac{ \tilde{B}(k, \cdot) }{ \| \tilde{B}(k, \cdot)\|} - \frac{ \tilde{B}(k', \cdot) }{ \| \tilde{B}(k', \cdot)\|} \right\|  \geq \psi_0 \tilde{\gamma} \,. \]
Thus the distance based clustering subroutine $\mathcal D$ used in Algorithm 1', such as the minimum spanning tree, can correctly cluster all nodes, i.e., $\hat{g}^{(2)} = g^{(2)}$, if
 \[ \sup_{v,v' \in \mathcal V_2, g_v = g_{v'}} \left\| \frac{\hat{h}_v}{ \| \hat{h}_v \| }  - \frac{\hat{h}_{v'}}{ \| \hat{h}_{v'} \| }  \right\|  < \inf_{v,v' \in \mathcal V_2, g_v \neq g_{v'}}  \left\| \frac{\hat{h}_v}{ \| \hat{h}_v \| }  - \frac{\hat{h}_{v'}}{ \| \hat{h}_{v'} \| }  \right\| \,. \]

Therefore, we only need to show that with high probability, $\left\| \frac{\hat{h}_v}{ \| \hat{h}_v \| } - \frac{ \tilde{B}(g_v, \cdot) }{ \| \tilde{B}(g_v, \cdot)\|} \right\| \leq \delta $ for all nodes $v \in \mathcal V_2$, where 
\begin{equation}\label{ineq:DCdelta}
 0 < \delta < \frac{ \psi_0 \tilde{\gamma}}{4} \,.
 \end{equation}

 By \Cref{lem:DChvbound}, the approximation bound \eqref{eq:DCh-error} and inequality \eqref{ineq:DCdelta} hold with high probability if we choose $a$ to be a constant larger than $4 q_0 / (\psi_0 \tilde{\gamma})$, where $q_0 = q_0(\pi_0, \psi_0, B_0, \tilde \gamma \psi_0/4)$ as specified in \Cref{lem:DChvbound}.
\end{proof}

\begin{lemma}[Lower bound of the distances between normalized rows of $\tilde{B}$] \label{lem:DCrowdiff}
If a degree corrected block model satisfies Assumptions A1' and A4, and $\tilde{B}$ is defined as in equation \eqref{def:tildeB}, then
$$
\min_{1<k<k'\le K}\left\|
\frac{ \tilde B (k,\cdot)}{\|\tilde B (k,\cdot)\|}-
\frac{ \tilde B (k',\cdot)}{\|\tilde B (k',\cdot)\|}
\right\|\ge \psi_0\tilde\gamma\,.
$$
\end{lemma}

\begin{proof}
Define matrix
\[ \Psi =  \textrm{diag}\left(  \frac{ \sum_{v' \in \mathcal I_1^{(1)}} \psi_{v'}}{ | \mathcal I_1^{(1)} | }, ..., \frac{ \sum_{v' \in \mathcal I_K^{(1)}} \psi_{v'}}{ | \mathcal I_K^{(1)} | } \right) \,. \]

We only need to prove that $\left\|\frac{\Psi B_0(k, \cdot)^T}{ \| \Psi B_0(k, \cdot)^T \|} - \frac{\Psi B_0(k', \cdot)^T}{ \| \Psi B_0(k', \cdot)^T \|}   \right\| \geq \psi_0 \tilde{\gamma}$, for any $k \neq k'$. 

Now we define
\begin{align*} w = & \frac{ B_0(k, \cdot)^T}{ \| \Psi B_0(k, \cdot)^T \|} - \frac{ B_0(k', \cdot)^T}{ \| \Psi B_0(k', \cdot)^T \|}
  =  u/s - v/t\,,
  \end{align*}
  where $u = \frac{B_0(k, \cdot)^T}{ \|  B_0(k, \cdot)^T \|}$, 
  $v = \frac{B_0(k', 
  \cdot)^T}{ \|  B_0(k', \cdot)^T \|}$, $s=\frac{\|  \Psi B_0(k, \cdot)^T \|}{\| B_0(k, \cdot)^T \|}$, and $t=\frac{\|  \Psi B_0(k', \cdot)^T \|}{\| B_0(k', \cdot)^T \|}$.
By Assumption A4, we have
\[
 \psi_0  \le \frac{\|  \Psi B_0(k, \cdot)^T \|}{\| B_0( k, \cdot)^T \|} \leq 1,~~
\forall~k \,. 
\]
Thus,
\[ \| w \| \geq \min_{\psi_0 \leq s,t \leq 1} \left\| \frac{u}{s} - \frac{v}{t}  \right\| \,.  \]

Because $u$ and $v$ are two unit vectors with $u^T v \ge 0$, 
it is straightforward to check that the function
\[ f(t,s) = \left\| \frac{u}{s} - \frac{v}{t}  \right\|^2 = \frac{1}{t^2} + \frac{1}{s^2} - \frac{2}{ts} u^Tv, \quad \psi_0 \leq t,s \leq 1 \]
 reaches its minimum $\| u-v\|^2$, when $t=s=1$. Therefore,
\[ \| w \| \geq \| u - v \| = \left\|  \frac{B_0(k, \cdot)^T}{ \|  B_0(k, \cdot)^T \|} - \frac{B_0( k', \cdot)^T}{ \|  B_0( k', \cdot)^T \|} \right\| \geq \tilde{\gamma} \,. \]

Using the fact that smallest eigenvalue of $\Psi$ satisfies $\lambda_{\min} (\Psi) \geq \psi_0$, we have
\[  \left\| \frac{ \Psi B_0( k, \cdot)^T }{  \| \Psi B_0( k, \cdot)^T \| } -  \frac{ \Psi B_0( k', \cdot)^T }{  \| \Psi B_0( k', \cdot)^T \| } \right\| = \left\|  \Psi w \right\| \geq \psi_0 \|w\| \geq \psi_0 \tilde{\gamma}  \,. \]
\end{proof}

\begin{lemma}\label{lem:DChvbound}
Given $\mathcal V_1$, $\mathcal V_2$, $g^{(1)}$, and $\hat g^{(1)}$ satisfying 
the conditions of \Cref{lem:crossclustsphere}, let $\alpha_n = (q / \delta)
 \log n / n$. For any $\delta > 0$, there exists $q_0 = q_0(\pi_0, \psi_0, B_0, \delta)$, such that if $q \geq q_0$,
\[ \mathbb{P} \left(  \left\| \frac{\hat{h}_v}{ \| \hat{h}_v \| } - \frac{ \tilde{B}(g_v, \cdot) }{ \| \tilde{B}(g_v, \cdot)\|} \right\| \le \delta, ~~\forall~v \in \mathcal V_2  \right) = 1 - O(n^{-1})\,, \] 
where the probability is conditional on $\mathcal V_1, \mathcal V_2$ and $\hat{g}^{(1)}$.
\end{lemma}

\begin{proof}
Let $a = q / \delta$ and
$L= \min_k \| B_0(k, \cdot) \|>0$.  
First, by the definition of $\tilde{B}$, we have
\[ \max \left\{  \| \hat{h}_v \| ,\| \psi_v \tilde{B}(g_v, \cdot)\|  \right\}  \geq \| \psi_v \tilde{B}(g_v, \cdot)\| \geq \psi_0^2 \alpha_n \min_j \| B_0(j, \cdot) \| = \frac{\psi_0^2 L q}{\delta} \frac{\log n}{n} \,,  \]
Therefore,
\[\begin{split}
 \left\| \frac{\hat{h}_v}{ \| \hat{h}_v \| } - \frac{ \tilde{B}(g_v, \cdot) }{ \| \tilde{B}(g_v, \cdot)\|} \right\| &= \left\| \frac{\hat{h}_v}{ \| \hat{h}_v \| } - \frac{ \psi_v \tilde{B}(g_v, \cdot) }{ \| \psi_v \tilde{B}(g_v, \cdot)\|} \right\| \\
 & \leq 2 \frac{ \| \hat{h}_v - \psi_v \tilde{B}(g_v, \cdot)  \| }{ \max\{  \| \hat{h}_v \| ,\| \psi_v \tilde{B}(g_v, \cdot)\|  \} } \\
 & \leq \frac{2 \delta }{\psi_0^2 L q} \frac{n}{ \log n}  \| \hat{h}_v - \psi_v \tilde{B}(g_v, \cdot)  \| \,.
 \end{split}\]
So we only need to bound 
\[ \mathbb{P} \left(\| \hat{h}_v - \psi_v \tilde{B}(g_v, \cdot) \| \ge \frac{\psi_0^2 L q}{2} \frac{\log n}{n} \right) \leq 
\sum_{k=1}^K \mathbb{P} \left( \left| \hat{h}_{v,k} - \psi_v \tilde{B}(g_v, k) \right| \ge \frac{\psi_0^2 L q }{2\sqrt{K}} \frac{\log n}{n} \right) \,, \] and the rest
of the proof follows by adapting that of \Cref{lem:individualhv}.  The details
are given below.

Since inequalities \eqref{ineq:I_kbound}, \eqref{eq:I_hat_bound} in 
\Cref{lem:individualhv} still hold, for any $k$, we have
\[\begin{split}
\left| \hat{h}_{v,k} -  \psi_v \tilde{B}(g_v, k) \right| & \leq \left| \frac{ \sum_{ v' \in \hat{\mathcal I}_k^{(1)} } A_{v,v'} - \sum_{v' \in \mathcal I_k^{(1)}} A_{v, v'} }{ |\hat{\mathcal I}_k^{(1)}| } \right| + \left| \frac{ \sum_{ v' \in \mathcal I_k^{(1)} } A_{v,v'} }{ |\hat{\mathcal I}_k^{(1)}| } -  \frac{ \sum_{ v' \in \mathcal I_k^{(1)} } A_{v,v'} }{ |\mathcal I_k^{(1)}| } \right| \\ 
& \qquad + \left|  \frac{ \sum_{ v' \in \mathcal I_k^{(1)} } A_{v,v'} }{ |\mathcal I_k^{(1)}| }  - \frac{ \sum_{v' \in \mathcal I_k^{(1)}} \psi_{v'}}{ | \mathcal I_k^{(1)} | } \psi_v B(g_v , k) \right| \\
& \leq \frac{ \left| \sum_{ v' \in \hat{\mathcal I}_k^{(1)} } A_{v,v'} - \sum_{v' \in \mathcal I_k^{(1)}} A_{v, v'} \right| }{ \pi_0 n / 2 } + \frac{ \left|  | \mathcal I_k^{(1)}| -  |\hat{\mathcal I}_k^{(1)}|  \right| }{ \pi_0^2 n^2 / 2 }  \sum_{ v' \in \mathcal I_k^{(1)} } A_{v,v'} \\
& \qquad +  \left|  \frac{ \sum_{ v' \in \mathcal I_k^{(1)} } A_{v,v'} }{ |\mathcal I_k^{(1)}| }  -  \frac{ \sum_{v' \in \mathcal I_k^{(1)}} \psi_{v'}}{ | \mathcal I_k^{(1)} | } \psi_v B(g_v, k) \right|  \\
& = T_1 + T_2 + T_3 \,.
\end{split}\]

Now we only need to bound the three terms individually. First by \eqref{ineq:I_kbound} we know
$|\{ v': \hat{g}_{v'}^{(1)} \neq g_{v'}^{(1)} \}| \leq n / f(\pi_0 a \log n)$. Then using Bernstein's inequality we have for $n$ large enough,
\[\begin{split}
\mathbb{P} \left( T_1 \ge \frac{\psi_0^2 L q}{6 \sqrt{K}} \frac{\log n}{n} \right) & \leq 
\mathbb{P} \left( \sum_{v': \hat{g}_{v'}^{(1)} \neq g_{v'}^{(1)}} A_{v,v'} \ge
 \frac{\pi_0 n}{2}  \frac{\psi_0^2 L q}{6 \sqrt{K}}  \frac{\log n}{n} \right) \\
& \leq n^{-\frac{c_1'^2(\pi_0,\psi_0, K,L)q^2}{a/f(\pi_0 a \log n )+c_1'(\pi_0,\psi_0, K,L)q/3}} \le
 n^{-c_1'(\pi_0,\psi_0, K,L)q}\,,
\end{split}\]
where $c_1'(\pi_0,\psi_0, K,L)$ is a positive constant depending only on $\pi_0$, $\psi_0$, $K$ and $L$.

To control $T_2$, similarly we have, for $n$ large enough, using Bernstein's inequality,
\[\begin{split}
\mathbb{P} &\left( T_2 \ge \frac{\psi_0^2 L q}{6\sqrt{K}} \frac{\log n}{n}  \right) \leq \mathbb{P} \left( 
\sum_{ v' \in \mathcal I_k^{(1)} } A_{v,v'} \ge
 \frac{\pi_0^2 n^2}{2\left|  | \mathcal I_k^{(1)}| -  |\hat{\mathcal I}_k^{(1)}|  \right|}  \frac{\psi_0^2 L q}{6\sqrt{K}} \frac{\log n}{n} \right) \\
& \qquad \leq  \mathbb{P} \left( 
 \sum_{ v' \in \mathcal I_k^{(1)} }  A_{v,v'}  \ge
 \frac{L \psi_0^2 \pi_0^2 q \log n f(\pi_0a\log n)}{12\sqrt{K}}
   \right) \\
& \qquad \leq n^{-\frac{c_2'^2(\pi_0, \psi_0, K, L)f^2(\pi_0 a\log n) q^2}{a+c_2'(\pi_0, \psi_0, K, L)f(\pi_0 a\log n) q/3}}\le n^{-c_2'(\pi_0, \psi_0, K, L) q}\,,
\end{split}\]
where $c_2'(\pi_0, \psi_0, K, L)$ is another positive constant function of $\pi_0$, $\psi_0$, $K$ and $L$.

Directly applying Bernstein's inequality to $T_3$, we have
\[\begin{split}
 \mathbb{P} \left( T_3 > \frac{\psi_0^2 L q}{6\sqrt{K}} \frac{\log n}{n} \right) &  \le
  \mathbb{P} \left( \bigg|\sum_{ v' \in \mathcal I_k^{(1)} } [ A_{v,v'}  - \psi_v \psi_{v'} B(g_v, k) ]\bigg|\ge  \frac{\psi_0^2 L q}{6\sqrt{K}}\frac{\log n}{n} \pi_0 n \right) \\
 & \leq 2 n^{-c_3'(\pi_0, \psi_0, K, L, \delta) q} \,, 
 \end{split}\]
where $c_3'(\pi_0, \psi_0, K, L, \delta)$ is also a positive constant function independent of $n$.

The claimed result follows by choosing $q$ large enough and applying union
bound over all $v\in\mathcal V_2$.
\end{proof}


%


\bibliographystyle{apa-good}
\bibliography{network}
\end{document}